\documentclass[lettersize,journal]{IEEEtran}
\usepackage{amsmath,amsfonts}
\usepackage{array}
\usepackage[caption=false,font=normalsize,labelfont=sf,textfont=sf]{subfig}
\usepackage{textcomp}
\usepackage{stfloats}
\usepackage{url}
\usepackage{verbatim}
\usepackage{graphicx}
\usepackage{cite}
\usepackage{setspace}
\usepackage{lmodern}
\usepackage[numbers]{natbib}  
\usepackage{bm}
\usepackage{comment}
\usepackage{hyperref}
\usepackage{url}
\usepackage{bbm}
\usepackage{amsthm}
\usepackage{makecell}
\usepackage{graphicx}
\usepackage{microtype}
\usepackage{graphicx}
\usepackage{hyperref}
\usepackage{algorithmic}
\usepackage[ruled,vlined]{algorithm2e}
\usepackage{mathtools}
\usepackage{unicode}
\usepackage{booktabs}
\usepackage{multirow}
\usepackage{csquotes}
\usepackage{subcaption}
\usepackage{rotating}
\usepackage{colortbl}
\usepackage{mathrsfs}
\usepackage{xcolor}
\usepackage{threeparttable}
\newtheorem{theorem}{Theorem}[section]

\theoremstyle{definition}

\theoremstyle{remark}

\newcommand{\Mod}[1]{\ (\mathrm{mod}\ #1)}

\begin{document}

\title{Congruence Decomposition with Neural Block Solvers for Large-Scale PCI Assignment}

\author{
	Yeqing Qiu, Chengpiao Huang, Ye Xue, Akang Wang, Fan Xu, Zhipeng Jiang, Dong Zhang, Ruoyu Sun, Qingjiang Shi, and Zhi-Quan Luo
        \thanks{The work was supported in part by the National Key Research and Development Program of China under Grant 2022YFA1003900, the National Natural Science Foundation of China under Grant 62301334, Guangdong Major Project of Basic and Applied Basic Research under Grant 2023B0303000001, and the Guangdong Provincial Key Laboratory of Big Data Computing.\\
        {\em (Corresponding authors: Qingjiang Shi, Zhi-Quan Luo)}
        \\
        Yeqing Qiu is with the School of Science and Engineering, The Chinese University of Hong Kong, Shenzhen, China, and Shenzhen Research Institute of Big Data, Shenzhen, China (email: yeqingqiu@link.cuhk.edu.cn). 

        Chengpiao Huang is with the Department of Industrial Engineering and Operations Research, Columbia University, New York, USA (email: chengpiao.huang@columbia.edu). 
        
        Ye Xue is with the School of Intelligent Systems Engineering, Sun Yat-sen University, Shenzhen, China (email: xuey57@mail.sysu.edu.cn). 

        Fan Xu is with the College of Electronic and Information Engineering, Tongji University, Shanghai, China (email: xxiaof999@tongji.edu.cn). 
        
        Zhipeng Jiang is with the School of Mathematical Sciences, University of Chinese Academy of Sciences, Beijing, China (email: jiangzhipeng@ucas.ac.cn). 
        
        Dong Zhang is with Huawei Technologies, Shenzhen, China (email: zhangdong48@huawei.com). 

        Akang Wang and Ruoyu Sun are with Shenzhen Research Institute of Big Data, Shenzhen, China, and School of Data Science, The Chinese University of Hong Kong, Shenzhen, China (email: wangakang@sribd.cn, sunruoyu@cuhk.edu.cn). 

        Qingjiang Shi is with the School of Computer Science and Technology, Tongji University, Shanghai, China, and Shenzhen Research Institute of Big Data, Shenzhen, China (email: shiqj@tongji.edu.cn).

        Zhi-Quan Luo is with The Chinese University of Hong Kong, Shenzhen, China, and Shenzhen Research Institute of Big Data, Shenzhen, China (email: luozq@cuhk.edu.cn).
        }
	}

\markboth{Under Review}%
{Shell \MakeLowercase{\textit{et al.}}: A Sample Article Using IEEEtran.cls for IEEE Journals}


\maketitle

\begin{abstract}

Physical Cell Identity (PCI) assignment is essential for interference management in dense 5G networks. As cellular networks scale, PCI reuse becomes unavoidable, which may cause collisions, confusions, and multiple forms of modular interference. Jointly mitigating these effects gives rise to a large-scale, multi-objective combinatorial optimization problem that is difficult to solve efficiently at practical network scales. In this work, we propose a congruence decomposition framework with neural block solvers for large-scale PCI assignment. The proposed decomposition exploits the arithmetic structure of PCI values to decouple multiple modular interference objectives into a collection of blockwise Min-\(k\)-Partition subproblems, followed by a graph coloring procedure to resolve PCI conflicts. For the resulting NP-hard Min-\(k\)-Partition subproblems, we develop neural block solvers by parameterizing their relaxed quadratic formulations with graph neural networks, enabling efficient optimization at large scales. Discrete assignments are recovered through conditional expectation rounding with theoretical guarantees. Experiments on synthetic cellular graphs and real-world 5G networks show that the proposed method consistently outperforms existing modular-interference-aware baselines in modular interference reduction, conflict elimination, and computational efficiency.

\end{abstract}


\section{Introduction}

The advancement of 5G mobile networks has demonstrated great potential to transform wireless communications, enabling ultra-dense deployments, spectrum agility, and diverse application scenarios~\cite{mansoor2017tutorial, ge2016ultra, luolou2015enhanced, xu2021survey}.
However, practical deployments have revealed a persistent gap between theoretical capabilities and actual user experience, especially in dense urban environments~\cite{pierucci2015quality}. 
This gap arises not only from physical infrastructure limitations but also from the lack of intelligent coordination and adaptive parameter configuration~\cite{luo2023srcon, liu2024}. 
As a result, effective optimization of key network parameters has become essential for interference management and performance improvement in large-scale 5G deployments.

\emph{Physical Cell Identity (PCI)} is one of the most important parameters in 5G networks. 
Each cell must be assigned a standardized integer PCI value used in synchronization and mobility management~\cite{3gpp38211, lodhi2023design}. 
Due to the limited range of usable PCI values and dense spatial reuse, improper assignments cause \emph{collisions}, \emph{confusions}, and \emph{modular interference}, which occur when pairs of cells are assigned identical PCI values or their assigned PCIs are congruent under specific moduli~\cite{qiu2025relaxation, gui2018pci2}. 
Jointly minimizing multiple modular interference objectives while eliminating collisions and confusions leads to a large-scale, multi-objective combinatorial optimization problem with coupled arithmetic and graph structures.

Early studies formulate PCI assignment as a graph coloring problem, where vertices represent cells and edges encode pairwise conflict relations. 
Greedy and randomized coloring heuristics~\cite{bandh2009graph,pratap2016randomized} are computationally efficient, but they only address collisions and confusions and do not model modular interference induced by arithmetic residue relations. 
Subsequent heuristic and metaheuristic methods, including genetic algorithms~\cite{panxing2016pci,shen2017novel} and memetic search~\cite{andrade2022physical}, incorporate modular penalties into the objective, but they remain problem-specific and provide limited structural understanding or theoretical guarantees. 
Binary quadratic and mixed-integer programming formulations~\cite{gui2018pci2,fairbrother2018two} improve modeling fidelity, but scale poorly to dense networks and do not explicitly exploit the arithmetic structure underlying multiple modular interference objectives.

Recent efforts on PCI optimization have progressed along two complementary directions: exploiting the arithmetic structure of modular interference and improving solver scalability.
On the structural side, Qiu et al.~\cite{qiu2025relaxation} exploited the Chinese Remainder Theorem (CRT) to decompose the joint mod-$3$ and mod-$30$ interference objectives into Min-$k$-Partition subproblems with $k=3$ and $10$. However, this construction relies on the specific coprime factorization $30=3\times 10$ and is tailored to this particular pair of modular objectives. It does not directly extend to PCI optimization with arbitrary modular interference objectives.
On the scalability side, Qiu et al.~\cite{11149326} proposed the Relaxed Gradient Projection (RGP) method, which relaxes the discrete PCI assignment problem and applies first-order optimization to the resulting continuous formulation. Subsequent work introduced Penalized Mirror Descent (PMD) for solving the decomposed Min-$k$-Partition subproblems~\cite{qiu2025relaxation}, further improving computational efficiency. Nevertheless, these methods optimize directly in the node-wise assignment space and still incur substantial computational cost on large instances.

To further improve scalability, we consider neural parametrization as an alternative.
Rather than treating node-wise assignment variables as free optimization variables, a graph neural network (GNN) generates the relaxed assignments through graph-based message passing. 
This parametrization induces graph-structured dependencies among node assignments and is naturally amenable to parallel computation on large graphs. Recent advances in neural-parametrized optimization have demonstrated the effectiveness of this approach for graph-structured combinatorial problems, including Max-$k$-Cut~\cite{qiu2024ros}, QUBO~\cite{schuetz2022combinatorial}, and scheduling~\cite{9962800}. These developments motivate using GNN parametrization as a scalable block solver for the large Min-$k$-Partition subproblems arising from PCI decomposition.

In this work, we propose a congruence decomposition framework with neural block solvers that jointly address modular interference, collisions, and confusions in large-scale PCI assignment. 
Our key contributions are summarized as follows.

\begin{itemize}
\item {\bf Congruence Decomposition for Multi-Modular PCI Assignment.}  
We develop a congruence decomposition framework for PCI assignment with multiple modular interference objectives and hard collision and confusion constraints. By exploiting arithmetic relations among the involved moduli, the proposed framework decomposes the coupled modular objectives into blockwise Min-$k$-Partition subproblems, followed by a graph-coloring stage to handle collision and confusion constraints.

\item {\bf Neural Block Solver for Min-$k$-Partition.}  
For the resulting NP-hard Min-\(k\)-Partition subproblems, we develop a neural block solver that parameterizes the simplex-relaxed solution using a GNN and optimizes the neural parameters against the corresponding quadratic objective. A deterministic conditional expectation rounding scheme then recovers a discrete assignment without increasing the relaxed objective value.

\item {\bf Evaluation on Synthetic and Real-World Cellular Networks.} 
We evaluate the proposed framework on synthetic random geometric graphs (RGGs) and real-world measurement reports from a commercial cellular network. 
Across both settings, our method simultaneously reduces collisions and confusions while achieving the lowest multi-modular interference among competitive baselines. 
On a real-world network with thousands of cells, it further reduces modular interference and achieves a $2.04\times$ speedup over the strongest decomposition-based baseline, demonstrating its effectiveness for large-scale PCI optimization.

\end{itemize}

\emph{Synopsis.} The subsequent sections of the paper are structured as follows. Section \ref{sec:prob} introduces the system model and formulates the PCI assignment problem. 
Section \ref{sec:decomposition} presents the congruence decomposition framework for multi-modular PCI assignment.
Building upon this decomposition, Section \ref{sec:ros} presents the neural block solver for the resulting Min-$k$-Partition subproblems, together with a deterministic conditional expectation rounding scheme for recovering the discrete solutions. 
Section \ref{sec:exp} presents numerical experiments on synthetic and real-world data. Finally, Section \ref{sec:con} concludes the paper.

\emph{Notations.} 
We denote by $\mathbb{Z}$ and $\mathbb{R}$ the sets of integers and real numbers, respectively, and by $\mathbb{Z}_+$ the set of non-negative integers. For any positive integer $k$, $\mathbb{Z}_k = \{0,1,\dots,k-1\}$, and $\mathbb{Z}_k^N$ denotes the set of $N$-dimensional vectors with entries in $\mathbb{Z}_k$. 
For $x\in\mathbb{R}$, $\lfloor x\rfloor$ is the largest integer not exceeding $x$, and for any positive integer $x$, $[x] = \{1,2,\dots,x\}$.
For $a,b\in\mathbb{Z}$ and any positive integer $k$, we write $a\equiv b \pmod{k}$ if $a-b$ is divisible by $k$. $\gcd(a, b)$ and $\mathrm{lcm}(a, b)$ denote the greatest common divisor and least common multiple of $a$ and $b$, respectively.
We use lowercase and uppercase boldface letters for column vectors and matrices, respectively. For a matrix $\bm{A}$, $A_{i,j}$ is its $(i,j)$-th entry, $\bm{A}_{i\cdot}$ is the $i$-th row, $\bm{A}^\top$  is the transpose, and $\mathrm{Tr}(\bm{A})$ is the trace. Finally, $\mathbbm{1}\{A\}$ is the indicator of event $A$, and $\vert \mathcal{S} \vert$ denotes the cardinality of a set $\mathcal{S}$. 
\begin{table*}[ht]
\centering
\renewcommand{\arraystretch}{1.2} 
\setlength{\tabcolsep}{12pt} 
\begin{tabular}{@{}cll@{}}
\toprule
\textbf{Mod-$p$} & \textbf{Typical interference source} & \textbf{Reference} \\
\midrule
mod-3  & Primary and Secondary Synchronization Signal reuse & \cite{3gpp38211,acedo2015analysis} \\
mod-4  & Physical Broadcast Channel Demodulation Reference Signal reuse & \cite{3gpp38211, andrade2022physical} \\
mod-6  & Uplink and Downlink Demodulation Reference Signal pattern reuse & \cite{3gpp38211, gui2018pci2} \\
mod-30 & Zadoff–Chu root group reuse in PUCCH, PUSCH, SRS, and PRACH & \cite{3gpp38211,zeljkovic2022alpaca} \\
\bottomrule
\end{tabular}
\caption{Common modular interference types in 5G PCI assignment.}
\label{tab:common-k-extended}
\end{table*}

\section{Problem Formulation\label{sec:prob}}

\begin{figure}
    \centering
    \includegraphics[width=\linewidth]{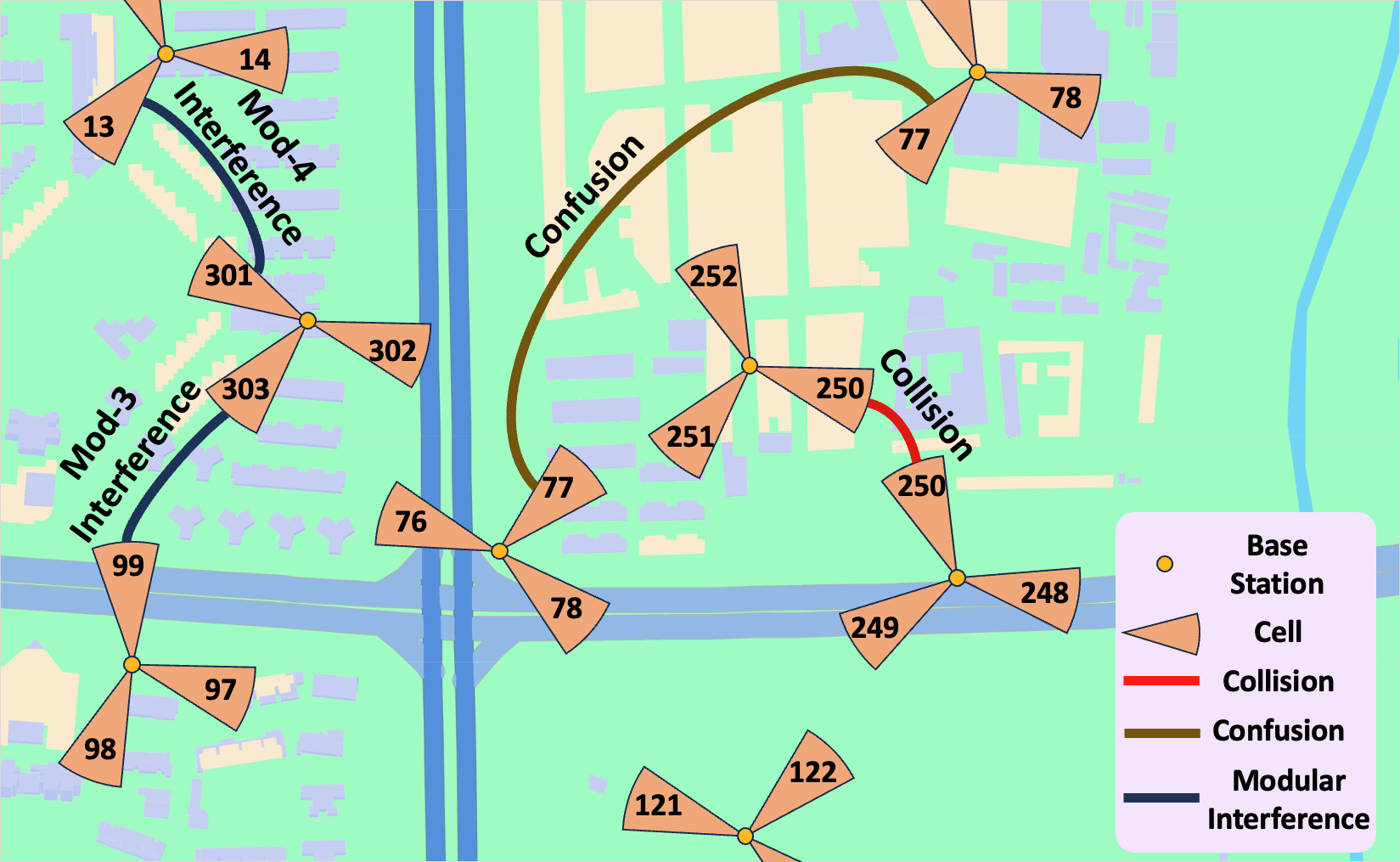}
    \caption{\small This figure illustrates the modular interference, collisions, and confusions of PCI assignment in a 5G wireless network, where each sector denotes a cell and the number inside indicates its assigned PCI value. 
    }
    \label{fig:sysmod}
\end{figure}

For 5G PCI assignment, we model a 5G wireless network with $N$ cells as a graph, where the vertices are given by $\mathcal{V}=\{1, 2, \cdots, N\}$, and vertex $i$ represents the $i$-th cell. Between cells, there is interference and a neighboring relation. The interference between cells is encoded in an \emph{interference matrix} $\bm{W}^{\rm{inter}}\in\mathbb{R}^{N\times N}$, where $W_{i,j}^{\rm{inter}}=W^{\rm{inter}}_{j,i} \ge 0$ is the interference between cell $i$ and cell $j$. For two cells $i$ and $j$, we say that $j$ is a \emph{neighboring cell} of $i$ if the signal of cell $j$ can be received by devices connected to cell $i$. The set of neighboring cell pairs is denoted by $\mathcal{E}_1\subseteq\mathcal{V}\times\mathcal{V}$. We further define the \emph{second-order neighboring set}
\[
\mathcal{E}_2 = \{ (i,j) \in \mathcal{V}\times\mathcal{V} \mid \text{$\exists\ell\in\mathcal{V}$ such that $(\ell,i),(\ell,j)\in\mathcal{E}_1$} \},
\]
which consists of cell pairs that share the same neighbor. We let $\mathcal{E}=\mathcal{E}_1\cup\mathcal{E}_2$. Similar to \cite{andrade2022physical}, we assume that $\bm{W}^{\rm{inter}}$, $\mathcal{E}_1$, and $\mathcal{E}_2$  have been estimated and are known.

The task of the 5G PCI assignment problem is to assign to each cell $i\in\mathcal{V}$ an integer $\text{PCI}_i\in\mathbb{Z}_{1008}$, called the \emph{PCI value}, to minimize \emph{modular interference} and eliminate \emph{collisions} and \emph{confusions}. The modular interference and PCI conflicts considered in this work are illustrated in Figure~\ref{fig:sysmod}, and are detailed as follows
\begin{itemize}
    \item \emph{Mod-$p$ interference} occurs between two cells $i$ and $j$, if their PCIs have the same mod $p$ value, i.e.~$\text{PCI}_i \equiv \text{PCI}_j \pmod{p}$. In this case, the magnitude of mod-$p$ interference between the two cells is given by $W^{\rm{inter}}_{i,j}$. Such interference can lead to potential resource conflicts and signal degradation during the channel-sounding process and negatively impact network performance. The common modular interferences are listed in Table~\ref{tab:common-k-extended}. For example, Figure~\ref{fig:sysmod} illustrates the mod-$3$ interference between cells with PCI $99$ and $303$, and the mod-$4$ interference between cells with PCI $13$ and $301$. 
    \item A \emph{collision} arises when two neighboring cells $i$ and $j$ share the same PCI, that is, $(i,j)\in\mathcal{E}_1$ and $\text{PCI}_i=\text{PCI}_j$. It prevents a UE from determining the serving cell and establishing a connection \cite{andrade2022physical}. For instance, two neighboring cells in Figure~\ref{fig:sysmod} are assigned PCI $250$, resulting in a PCI collision.
    \item A \emph{confusion} arises when two cells $(i,j)\in\mathcal{E}_2$ share the same PCI. In other words, this means that two cells $i,j$ sharing a common neighbor have the same PCI. This can potentially cause a UE confusion during cell handovers. For example, Figure~\ref{fig:sysmod} depicts a PCI confusion scenario where two non-neighboring cells share PCI $77$ and have a common neighbor with PCI $252$. 
\end{itemize}

Mathematically, the PCI assignment problem can be formulated as
\begin{subequations}    
\begin{align}
    \min_{\textbf{PCI}}& \quad 
    \Bigg( 
    \sum_{i,j=1}^NW^{\rm{inter}}_{i,j} \cdot \mathbbm{1}\{\text{PCI}_i\equiv\text{PCI}_j\Mod{p_1}\}, \label{obj-modp1} \\
    & \quad\ \  \sum_{i,j=1}^NW^{\rm{inter}}_{i,j} \cdot \mathbbm{1}\{\text{PCI}_i\equiv\text{PCI}_j\Mod{p_2}\} , \label{obj-modp2} \\
    & \quad\ \  \cdots \notag\\
    & \quad\ \  \sum_{i,j=1}^NW^{\rm{inter}}_{i,j} \cdot \mathbbm{1}\{\text{PCI}_i\equiv\text{PCI}_j\Mod{p_m}\} \Bigg), \label{obj-modpm} \\
    \text{s.t.}&\quad  \text{PCI}_i\neq\text{PCI}_j , \quad \forall (i,j)\in \mathcal{E}, \label{cons-coll-conf} \\
&\quad  \text{PCI}_i \in \mathbb{Z}_{1008}, \quad \forall i\in \mathcal{V}, \label{cons-range}
\end{align}\label{prob:pci}
\end{subequations}
where $\textbf{PCI} = (\text{PCI}_1,...,\text{PCI}_N)$.
This is a multi-objective optimization problem where the objectives are to minimize the modular interference \eqref{obj-modp1} -- \eqref{obj-modpm}, and the constraints include eliminating collisions and confusions \eqref{cons-coll-conf} and the PCI range requirement \eqref{cons-range}. In practice, a common set of mod values is $(p_1,p_2,p_3,p_4)=(3, 4, 6, 30)$ as shown in Table~\ref{tab:common-k-extended}.

\section{Congruence Decomposition for Multi-Modular PCI Assignment \label{sec:decomposition}}

This section develops a congruence decomposition framework for PCI assignment with multiple modular interference objectives. By exploiting the arithmetic relations among the involved moduli, the proposed framework decomposes the modular assignment stage into a sequence of blockwise Min-\(k\)-Partition subproblems, followed by graph coloring for handling collision and confusion constraints.
We refer to this construction as \emph{congruence decomposition}, as it represents each modular congruence through equality relations over a subset of prime-power coefficient blocks.

\subsection{Decomposition Framework}

Consider an integer variable $x_i$ associated with each node $i$. When the interference term depends on modular equalities $(x_i \bmod p_\ell)$ for multiple moduli $\{p_1, p_2, \dots, p_m\}$, all modular relations can be equivalently represented by $(x_i \bmod p)$ with $p = \mathrm{lcm}(p_1, \dots, p_m)$, since identical mod-$p$ values imply identical residues under every $p_\ell$. 
Our decomposition framework is based on a division-with-remainder reformulation. The optimization objectives depend on each cell's PCI through the mod $p_1,...,p_m$ values, which are uniquely determined by its mod $p=\mathrm{lcm}(p_1, \dots, p_m)$ value.
This motivates a standard division-with-remainder reformulation:
\begin{align*}
    \text{PCI}_i=p q_i+r_i,
\end{align*}
where $q_i=\left\lfloor\frac{\text{PCI}_i}{p}\right\rfloor\in \mathbb{Z}_{+}$ is the \emph{quotient}, and $r_i\in\mathbb{Z}_p$ is the mod $p$ value of $\text{PCI}_i$.

This representation naturally separates the modular component $r_i$, which determines all interference relations, from the quotient $q_i$, which enforces the remaining structural constraints. Substituting this representation into the general PCI optimization yields:
\begin{subequations}
\begin{align}
    \min_{\bm{q},\bm{r}} & \quad \Bigg(  \sum_{i,j=1}^N W^{\rm{inter}}_{i,j} \cdot \mathbbm{1}\{r_i\equiv r_j\Mod{p_1}\}, \notag\\
    & \quad\ \  \sum_{i,j=1}^N W^{\rm{inter}}_{i,j} \cdot \mathbbm{1}\{r_i \equiv r_j\Mod{p_2}\} , \notag\\
    & \quad\ \ \cdots, \notag\\
    & \quad\ \  \sum_{i,j=1}^N W^{\rm{inter}}_{i,j} \cdot \mathbbm{1}\{r_i \equiv r_j\Mod{p_m}\} \Bigg), \notag\\
    \text{s.t.}& \quad \bm{q}=[q_1,\cdots, q_N]\in \mathbb{Z}_+^N, \notag \\
    & \quad \bm{r}=[r_1,\cdots, r_N]\in \mathbb{Z}_{p}^N, \notag \\
    & \quad q_i \neq q_j \text{ or } r_i \neq r_j, \quad \forall (i,j)\in \mathcal{E} \notag\\
    & \quad p\times \bm{q} + \bm{r} \in \mathbb{Z}_{1008}^N. \notag
\end{align} 
\end{subequations}

Since all interference terms depend solely on $\bm{r}$, the problem decomposes into two sequential stages:
\begin{itemize}
    \item {\bf Sub-problem 1: Modular assignment.} Determine $\bm{r}$ to jointly minimize all modular interferences:
\begin{equation}
\begin{aligned}
     \min_{\bm{r}}&\quad \Bigg( \sum_{i,j=1}^N W^{\rm{inter}}_{i,j} \cdot \mathbbm{1}\{r_i\equiv r_j\Mod{p_1}\},   \\
            & \quad\ \  \sum_{i,j=1}^N W^{\rm{inter}}_{i,j} \cdot \mathbbm{1}\{r_i \equiv r_j\Mod{p_2}\} ,\\
            & \quad\ \ \cdots,  \\
     &\quad\ \ \sum_{i,j=1}^N W^{\rm{inter}}_{i,j} \cdot \mathbbm{1}\{r_i \equiv r_j\Mod{p_m}\} \Bigg),  \\
    \text{s.t.}&\quad  \bm{r}\in \mathbb{Z}_{p}^N.
\end{aligned}\label{subprob:mod_val_multi}
\end{equation}
    \item {\bf Sub-problem 2: Quotient assignment $\bm{q}$.} Given $\bm{r}$, assign $\bm{q}$ to eliminate residual conflicts while maintaining feasible ranges: 
    \begin{equation}
        \begin{aligned}
            \text{find}&\quad \bm{q}=(q_1, \cdots, q_N), \\
            \text{s.t.}&\quad q_i\neq q_j, \quad \forall (i, j)\in\mathcal{E}\text{ with } r_i = r_j , \\
            & \quad \bm{q}\in \mathbb{Z}_+^N, \\ 
            & \quad p \times \bm{q}+\bm{r}\in \mathbb{Z}_{1008}^N. 
        \end{aligned} \label{subprob:coloring}
    \end{equation}
\end{itemize}
This yields a clean modular–quotient decomposition, where $\bm{r}$ governs modular interference and $\bm{q}$ resolves remaining PCI conflicts.

\begin{figure*}[t]
    \centering
    \includegraphics[width=\linewidth]{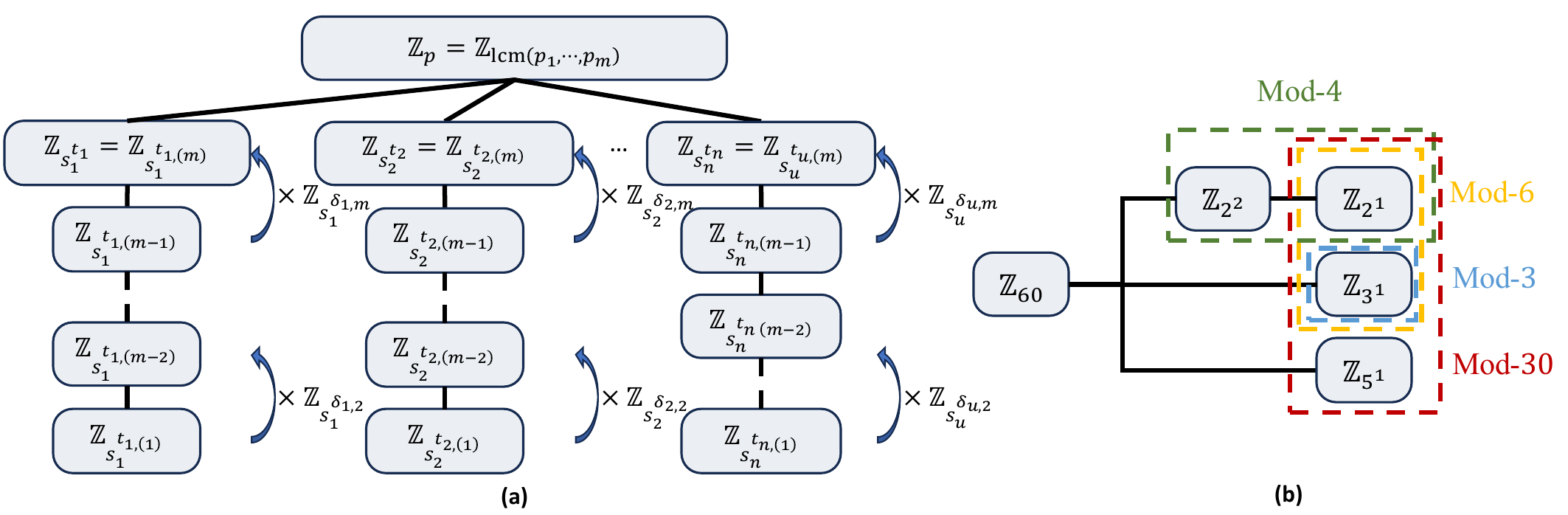}
    \caption{(a) Hierarchical congruence decomposition of $\mathbb Z_p$ into prime-power coefficient blocks. (b) The resulting multi-modular decomposition for PCI assignment simultaneously addressing mod-$3$, mod-$4$, mod-$6$, and mod-$30$ interferences. }
    \label{fig:PCI_decomp_full}
\end{figure*}
\subsection{Congruence Decomposition for Arbitrary Modular Interference Objectives}

We now study Subproblem 1 for arbitrary, possibly non-coprime moduli $\{p_\ell\}_{\ell\in[m]}$. 
While previous work~\cite{qiu2025relaxation} relied on the specific coprime factorization \(30=3\times10\) for the mod-\(3\) and mod-\(30\) objectives, we extend the decomposition to arbitrary modular objectives by combining the CRT with an $s_u$-adic decomposition that separates the contributions of different moduli within a common mod-$p$ representation.

We first consider the prime factorization of $p = \text{lcm}(p_1,\dots,p_m)$:
\[
p = s_1^{t_1}\cdots s_n^{t_n},
\]
where $s_1,...,s_n$ are the prime factors and $t_1,...,t_n\ge 1$. 
For any residue \(r\in\mathbb{Z}_p\), define its CRT coordinates by
\[
r^{(u)} := r \bmod s_u^{t_u}\in \mathbb{Z}_{s_u^{t_u}},
\qquad u\in[n].
\]
By the CRT, the mapping
\[
r \longleftrightarrow \big(r^{(1)},\ldots,r^{(n)}\big)
\in \prod_{u=1}^n \mathbb{Z}_{s_u^{t_u}}
\]
is bijective, and we write
\[
r=\mathrm{CRT}\big(r^{(1)},\ldots,r^{(n)}\big).
\]

We then examine how each modulus \(p_\ell\) depends on these CRT coordinates. Write
\[
p_{\ell} = s_1^{t_{1,\ell}} \cdots s_n^{t_{n,\ell}},
\]
where $0\le t_{u,\ell}\le t_u$ for each $u=1,...,n$. 

Then congruence modulo \(p_\ell\) is equivalent to congruence modulo \(s_u^{t_{u,\ell}}\) for every prime factor \(s_u\). 
Hence \(r \bmod p_\ell\) is fully determined by the truncated CRT coordinates
\[
r^{(u)} \bmod s_u^{t_{u,\ell}},\qquad u\in[n].
\]
This observation motivates an \(s_u\)-adic decomposition of each coordinate \(r^{(u)}\), so that different modular objectives can be expressed through different subsets of lower-order prime-power components.

We now decompose each mod-\(s_u^{t_u}\) coordinate. 
Let the exponents \(\{t_{u,\ell}\}_{\ell=1}^m\) of \(s_u\) be sorted as
\[
    t_{u,(1)}\le \cdots \le t_{u,(m)}=t_u ,
\]
and set \(t_{u,(0)}=0\). 
Then any residue \(r^{(u)}\in\mathbb{Z}_{s_u^{t_u}}\) admits the following \(s_u\)-adic block decomposition:
\begin{align}
    r^{(u)}=\sum_{h=1}^{m}r^{(u, h)}s_u^{t_{u, (h-1)}}, \label{eq:adic_based}
\end{align}
where $r^{(u, h)}\in\mathbb{Z}_{s_u^{\delta_{u, h}}}$ and $\delta_{u,h}=t_{u, (h)}-t_{u, (h-1)}$. 
When \(\delta_{u,h}=0\), the corresponding block is degenerate, i.e., \(\mathbb{Z}_{s_u^{\delta_{u,h}}}=\mathbb{Z}_1\), and thus carries no degree of freedom. 
The nondegenerate blocks group the \(s_u\)-adic digits of \(r^{(u)}\) according to the exponent intervals \((t_{u,(h-1)},t_{u,(h)}]\), with zero-length intervals corresponding to degenerate blocks.

Since the residue modulo \(p_\ell\) only depends on powers up to \(t_{u,\ell}\) for each prime factor \(s_u\), it is fully determined by the following coefficient blocks
\begin{align}
\left\{
r^{(u,h)}
\;\middle|\;
u\in[n],\ t_{u,(h)}\le t_{u,\ell}
\right\}, \label{def:coeff}
\end{align}
Equivalently, we define the corresponding block-index set as
\begin{align}
    \mathcal{J}_\ell=\{(u, h)\mid u\in[n], t_{u, (h)}\le t_{u, \ell}\} \label{def:block}
\end{align}
We formalize this coefficient-based characterization of modular equality in the following theorem.
\begin{theorem}[Blockwise Characterization of Modular Equality]
\label{thm:modular-block-equivalence}
For any two residues $r_i,r_j\in\mathbb{Z}_p$,  we have
\begin{align*}
r_i \equiv r_j \pmod{p_\ell}
\end{align*}
if and only if 
\begin{align*}
r_i^{(u,h)} = r_j^{(u,h)}, \qquad \forall (u,h)\in\mathcal{J}_\ell.
\end{align*}
\end{theorem}
\begin{proof}
See Appendix~\ref{proof:lamme:modular-block-equivalence}.
\end{proof}

Theorem~\ref{thm:modular-block-equivalence} shows that each modular congruence can be replaced by equality constraints on a subset of coefficient blocks. Accordingly, Subproblem 1 admits the following blockwise reformulation:
\begin{equation}
\begin{aligned}
\min_{\{\bm{r}^{(u, h)}\}} \quad&
\Bigg( \sum_{i,j=1}^N W^{\rm{inter}}_{i,j}\prod_{(u,h)\in\mathcal{J}_1} \mathbbm{1}\{r_i^{(u, h)} = r_j^{(u, h)}\}, \\  
& \sum_{i,j=1}^N W^{\rm{inter}}_{i,j}  \prod_{(u, h)\in\mathcal{J}_2} \mathbbm{1}\{r_i^{(u, h)} = r_j^{(u, h)}\}, \\
&\cdots, \\ 
&\sum_{i,j=1}^N W^{\rm{inter}}_{i,j}  \prod_{(u, h)\in\mathcal{J}_m} \mathbbm{1}\{r_i^{(u, h)} = r_j^{(u, h)}\} 
\Bigg), \\
\text{s.t.}\quad& \bm{r}^{(u, h)}\in\mathbb{Z}_{s_u^{\delta_{u, h}}}^N, \forall u\in [n], h\in[m].
\end{aligned}\label{eq:equi_problem}
\end{equation}

To solve this problem, we propose to update each $\bm{r}^{(u,h)}$ sequentially while fixing the other $\bm{r}^{(u^\prime,h^\prime)}$. This yields a weighted Min-$k$-Partition subproblem with $k=s_u^{\delta_{u, h}}$:
\begin{equation}
\begin{aligned}
\min_{\bm{r}^{(u,h)}} \quad & \sum_{\ell : (u,h)\in \mathcal{J}_\ell} \lambda_\ell \sum_{i,j=1}^N W^{\rm{inter}}_{i,j}\cdot \mathbbm{1}\{r_i^{(u,h)} = r_j^{(u,h)}\}  \\
&\cdot \prod_{\substack{(u^\prime,h^\prime)\in\mathcal{J}_\ell \\ (u^\prime,h^\prime)\ne(u,h)}} \mathbbm{1}\{r_i^{(u^\prime,h^\prime)} = r_j^{(u^\prime,h^\prime)}\}, \\[4pt]
\text{s.t.}\quad & \bm{r}^{(u,h)} \in \mathbb{Z}_{s_u^{\delta_{u,h}}}^N. 
\end{aligned}\label{eq:block-subproblem}
\end{equation}
Each update minimizes the conditional interference determined by its active blocks, where the weights $\{\lambda_\ell\}_{\ell\in[m]}$ encode the relative significance of different modular objectives.
This blockwise formulation defines a conditional partitioning problem: given the current assignments of other blocks, each subproblem searches for the optimal $k$-way partition of nodes in $\mathbb{Z}_{s_u^{\delta_{u,h}}}$ that most reduces modular interference.  
The procedure naturally admits a BCD scheme, whose convergence is guaranteed by the following theorem. 

\begin{theorem}[Finite-Step Termination of BCD]
\label{thm:convergence_BCD}
Consider the BCD procedure whose block update is given by \eqref{eq:block-subproblem}. 
If each block update is accepted only when it decreases its associated objective, 
then the overall procedure terminates in a finite number of updates.

\end{theorem}
\begin{proof}
    See Appendix \ref{proof:thm:convergence_BCD}.
\end{proof}

As a concrete specialization, consider the PCI assignment with
\[
(p_1,p_2,p_3,p_4)=(3,4,6,30),
\]
with $p=\operatorname{lcm}(3,4,6,30)=60=2^2\cdot 3\cdot 5$. The above construction then yields a four-block residue decomposition of $\mathbb Z_{60}$, shown in Fig.~\ref{fig:PCI_decomp_full} (b). In this representation, the mod-$3$ objective depends only on the ternary block, the mod-$4$ objective depends on the two binary blocks, the mod-$6$ objective depends on the binary and ternary blocks jointly, and the mod-$30$ objective depends on all blocks except the higher binary block. Hence the PCI interference criteria are converted into a nested family of blockwise equality relations, and each corresponding update reduces to a weighted Min-$k$-Partition subproblem, which will be addressed in Section~\ref{sec:ros}.

\subsection{Graph Coloring for Quotient Assignment}\label{sec:quotient}

We now revisit sub-problem~2 in \eqref{subprob:coloring}, which seeks a quotient assignment $\bm{q}$ given the mod-$p$ values $\bm{r}$ obtained from sub-problem 1.

Since collisions and confusions can occur only between constrained cell pairs with identical remainders, the constraints in \eqref{subprob:coloring} reduce to the edge set
\begin{align*}
    \mathcal{E}^\prime \coloneq \{(i,j)\in\mathcal{E} \mid r_i=r_j\}.
\end{align*}
Hence sub-problem 2 is equivalent to a \emph{graph coloring problem}~\cite{bandh2009graph} on the graph $(\mathcal{V},\mathcal{E}^\prime)$, where colors correspond to quotient values. We apply GGC~\cite{bandh2009graph,coloring} for efficient approximate solutions.

The remaining issue is the PCI range constraint $pq_i+r_i \in\mathbb{Z}_{1008}$. Following \cite{qiu2025relaxation}, we adopt a heuristic post-processing step: cells violating the range are reassigned feasible quotients that minimize additional collisions and confusions, processed in descending order of node degree (including second-order neighbors) to prioritize constrained regions. 

In summary, sub-problem 2 reduces to graph coloring on the graph $(\mathcal{V}, \mathcal{E}^\prime)$, complemented by a tailored adjustment procedure to enforce PCI feasibility.

\subsection{Summary}

Below, we summarize our framework for PCI assignment. 
The original problem is decomposed into modular sub-blocks, leading to two stages: 
\begin{enumerate}
    \item solving a sequence of blockwise Min-$k$-Partition subproblems that progressively minimize multi-modular interference, generalizing the graph-partition approach~\cite{qiu2025relaxation} beyond the coprime setting; 
    \item handling collisions and confusions via graph coloring under the PCI range constraint.
\end{enumerate}
The full procedure is outlined in Algorithm~\ref{alg:blockwise}. For brevity, we refer to the resulting multi-modular decomposition procedure as MMD.

\begin{algorithm}[t]
\KwData{Interference matrix $\bm{W}$, edge set $\mathcal{E}$, modulus set $\{p_1,\dots,p_m\}$, and weights $\{\lambda_1,\cdots,\lambda_m\}$.}
\KwResult{A PCI assignment $\textbf{PCI} \in \mathbb{Z}_{1008}^N$.}

\BlankLine
\textbf{(Pre-processing)} \\

\textbf{Compute} $p \leftarrow \mathrm{lcm}(p_1,\ldots,p_m)$\;

\textbf{Factorize} $p=\prod_{u=1}^n s_u^{t_u}$ and $p_\ell=\prod_{u=1}^n s_u^{t_{u,\ell}}$ for all $\ell\in[m]$\;

\For{$u=1$ \KwTo $n$}{
    \textbf{Sort} $\{t_{u,\ell}\}_{\ell=1}^m$ as $t_{u,(1)}\le \cdots \le t_{u,(m)}=t_u$\;

    \textbf{Set} $t_{u,(0)}=0$\;
    
    \For{$h=1$ \KwTo $m$}{
        \textbf{Set} $\delta_{u,h}\leftarrow t_{u,(h)}-t_{u,(h-1)}$\;
    
        \textbf{Set} $k_{u,h}\leftarrow s_u^{\delta_{u,h}}$\;
    }
}

\For{$\ell=1$ \KwTo $m$}{
\textbf{Set} \(\mathcal J_\ell \leftarrow \{(u,h)\mid u\in[n], h\in[m], t_{u,(h)}\le t_{u,\ell}\}\);
}

\BlankLine
\textbf{(Sub-problem 1: Mod-$p$ Assignment)} \\
\textbf{Initialize} $\bm{r}^{(u,h)} \leftarrow \bm{0}$ for all $(u,h)$\;

\Repeat{No improving block update}{
    \For{$u = 1$ \KwTo $n$}{
        \For{$h = 1$ \KwTo $m$}{
            \textbf{Update} $\bm{r}^{(u,h)}$ by solving Problem (\ref{eq:block-subproblem}) if the  objective decreases\;
        }
    }   
}

\textbf{Recover} $\bm{r} = \mathrm{CRT}\!\left( \sum_{h=1}^m \bm{r}^{(u,h)} \cdot s_u^{t_{u, (h-1)}} \right)_{u=1}^n$\;

\BlankLine
\textbf{(Sub-problem 2: Quotient Assignment)} \\
\textbf{Construct} restricted edge set $\mathcal{E}^\prime = \{(i,j)\in\mathcal{E}\mid r_i=r_j\}$\;

\textbf{Apply} GGC on $(\mathcal{V}, \mathcal{E}^\prime)$ to obtain quotient assignment $\bm{q}$\;

\If{some $pq_i + r_i > 1007$}{
    \textbf{Reassign} quotients of violating nodes using heuristic adjustment (process in descending order of degree)\;
}

\BlankLine
\textbf{Output} $\textbf{PCI} = p\bm{q} + \bm{r}$\;

\caption{Blockwise Alternating Minimization with Graph Coloring for PCI Assignment}
\label{alg:blockwise}
\end{algorithm}
\section{Neural Block Solver for Min-$k$-Partition \label{sec:ros}}

The multi-modular decomposition reduces the modular assignment stage to a sequence of block subproblems~\eqref{eq:block-subproblem}. 
For a block $\bm r^{(u,h)}\in\mathbb Z_{s_u^{\delta_{u,h}}}^N$, fixing all other blocks yields a weighted Min-$k$-Partition problem with $k=s_u^{\delta_{u,h}}$. 
This section develops an instance-wise GNN-parametrized solver for these block subproblems. We refer to the resulting neural block solver as GNN-Parametrized Optimization (GPO).

\subsection{Effective Block Graph Construction}

For each block $(u,h)$, we construct an effective weighted graph over $N$ nodes, with edge weights capturing all relevant interference contributions conditioned on other block assignments:
\begin{align*}
&\widetilde{\bm{W}}_{i,j}^{\rm{inter}} \\ &\quad = \sum_{\ell:(u,h)\in\mathcal{J}_\ell} \lambda_\ell \bm{W}_{i,j}
\cdot\prod_{\substack{(u^\prime,h^\prime)\in\mathcal{J}_\ell \\ (u^\prime,h^\prime)\ne (u,h)}} \mathbbm{1}\{ r_i^{(u^\prime,h^\prime)} = r_j^{(u^\prime,h^\prime)} \},
\end{align*}
where the product masks out node pairs whose modular equality has already been broken by the fixed active blocks.
The corresponding block problem is thus given by
\begin{equation}
\begin{aligned}
\min_{\bm{r}^{(u,h)}} &\quad \sum_{i=1}^N\sum_{j=1}^N \widetilde{\bm{W}}_{i,j}^{\rm{inter}} \, \mathbbm{1}\{ r_i^{(u,h)} = r_j^{(u,h)}\}, \\
\text{s.t. } & \quad \bm{r}^{(u,h)}\in \mathbb{Z}_k^N, 
\end{aligned}\label{eq:min-k-partition-block}
\end{equation}
with $k=s_u^{\delta_{u,h}}$.

For notational simplicity, we denote the effective block weight matrix $\widetilde{\bm{W}}^{\rm{inter}}$ by a general $\bm W$ and encode $\bm r^{(u,h)}$ as a one-hot matrix $\bm X\in\{0,1\}^{N\times k}$. 
Then~\eqref{eq:min-k-partition-block} is rewritten as a quadratic objective over one-hot vectors:
\begin{align}
\min_{\bm X} \quad &  \mathrm{Tr}(\bm X^\top \bm W\bm X), \label{eq:onehot-block}\\
\mathrm{s.t.}\quad 
& \bm X\in\mathcal X \coloneq \{\bm e_1,\ldots,\bm e_k\}^N .
\notag
\end{align}

\subsection{GNN-Parametrized Relaxed Optimization \label{sec:GNN_optimization}}

Given the challenges associated with solving the discrete problem \eqref{eq:onehot-block}, we take the natural relaxation approach of replacing $\mathcal{X}$ by its convex hull, which is the Cartesian product of $N$ $k$-dimensional probability simplices, denoted by $\Delta_k^N$. Consequently, the discrete problem \eqref{eq:onehot-block} is relaxed into the following continuous optimization form:
\begin{align}
\min_{\bm{X}}& \quad  \mathrm{Tr}(\bm{X}^\top \bm{W} \bm{X})\coloneq f(\bm{X}), \label{eq:soft-relaxation}\\
\text{s.t.} & \quad \bm{X}\in\Delta_k^N \subseteq [0,1]^{N\times k}.\notag
\end{align}

To solve this relaxed block problem~\eqref{eq:soft-relaxation}, we use a GNN as a graph-aware neural parametrization of the relaxed assignment. 
The Min-$k$-Partition block naturally aligns with a node classification task where each node is assigned one of $k$ labels. 
Moreover, the gradient of the quadratic objective satisfies $\nabla f(\bm X)=2\bm W\bm X$, which has the form of weighted neighborhood aggregation. 
This matches the message-passing structure of GNNs and motivates representing the relaxed assignment through a graph-structured neural model.

Specifically, each node $i$ is initialized with a trainable embedding $\bm h_i^{(0)}$ with Xavier initialization~\cite{xavier2010understanding}. 
We adopt an $L$-layer GNN~\cite{morris2019weisfeiler} with the update
\begin{align}
\bm h_i^{(l)}
=
\sigma\left(
\bm\Phi_1^{(l)}\bm h_i^{(l-1)}
+
\bm\Phi_2^{(l)}
\sum_{j\in\mathcal N(i)}
W_{j,i}\bm h_j^{(l-1)}
\right)
\end{align}
for $l=1,\cdots, L$ with trainable parameters $\bm{\Phi}_1^{(l)}, \bm{\Phi}_2^{(l)}$ and activation function $\sigma$. 

This update combines node-wise transformations with weighted aggregation from neighboring nodes, allowing the parametrization to exploit the graph structure of the current block instance.
After the final row-wise softmax activation, we denote the resulting simplex-valued output by
\[
\bm H_{\bm\Phi}^{(L)}\in\Delta_k^N.
\]
Each row of the final embeddings $\bm{H}^{(L)}_{\bm{\Phi}}$ lies in $\Delta_k^N$, producing a feasible relaxed assignment $\bm{X} = \bm{H}^{(L)}_{\bm{\Phi}}$, where $\bm{\Phi}\coloneq(\bm{H}^{(0)}, \bm{\Phi}_1^{(1)}, \bm{\Phi}_2^{(1)}, \cdots, \bm{\Phi}_1^{(L)}, \bm{\Phi}_2^{(L)})$ denotes the trainable parameters of GNN.

When solving a testing instance $\bm{W}$, we optimize the trainable parameters from a fresh random initialization using the objective
\begin{align}
\min_{\bm{\Phi}}\quad \mathcal{L}(\bm{\Phi}) \coloneq \mathrm{Tr}\left(\bm{H}_{\bm{\Phi}}^{(L)\top}\bm{W}\bm{H}_{\bm{\Phi}}^{(L)}\right).\label{eq:gnn-parametrization}
\end{align}
The parameters are optimized separately for each weighted block instance, allowing the parametrization to adapt to the effective graph induced by the current modular decomposition.
After obtaining the instance-specific parameters $\bm{\Phi}^\star$, the relaxed solution is produced by a single forward evaluation:
\begin{align*}
\overline{\bm{X}}=\bm{H}_{\bm{\Phi}^\star}^{(L)}.
\end{align*}

\subsection{Conditional Expectation Rounding \label{sec:random_sampling}}

Let $\overline{\bm{X}}\in\Delta_k^N$ be a feasible solution to the relaxed problem (\ref{eq:soft-relaxation}). 
Our goal is to construct a discrete assignment $\bm{X} \in \mathcal{X}$ whose objective value is no larger than $f(\overline{\bm X})$.

We first introduce an auxiliary random matrix $\bm{Z}$ taking values in $\mathcal{X}$, where each row is independently distributed as
\begin{align}
\mathbb P(\bm{Z}_{i\cdot}=\bm e_s)=\overline X_{i,s},
\qquad i\in[N],\ s\in[k].
\label{eq:aux-random-assignment}
\end{align}
Since $\bm W$ has zero diagonal, this product distribution satisfies
\begin{align}
\mathbb E[f(\bm{Z})]=f(\overline{\bm X}).
\label{eq:unbiased-rounding}
\end{align}
Thus, the relaxed objective equals the expected objective of the randomized discrete assignment~\cite{qiu2024ros}. 
However, drawing realizations from this randomized rounding distribution only gives this guarantee in expectation and may require many trials to obtain a realization with objective no larger than $f(\overline{\bm X})$.

We therefore derandomize this procedure by conditional expectation. 
At step $i$,  suppose the first $i-1$ rows have been fixed as 
$\bm X_{1\cdot},\ldots,\bm X_{(i-1)\cdot}$. 
Define the conditioning event
\begin{align}
\mathcal C_{i-1}
:=
\{\bm{Z}_{a\cdot}=\bm X_{a\cdot},\ a=1,\ldots,i-1\}.
\label{eq:conditioning-event}
\end{align}
For each candidate label $s\in[k]$, define
\begin{align}
\psi_i(s)
:=
\mathbb E\!\left[
f(\bm{Z})
\,\middle|\,
\mathcal C_{i-1},\ \bm{Z}_{i\cdot}=\bm e_s
\right].
\label{eq:conditional-cost}
\end{align}
The algorithm fixes node $i$ to the label minimizing $\psi_i(s)$. The description is given in Algorithm~\ref{alg:conditional}. 
Below we show that this method always leads to a solution $\bm{X}$ no worse than $\overline{\bm{X}}$.

\begin{algorithm}[t]
\caption{Conditional Expectation Rounding}
\label{alg:conditional}
\begin{algorithmic}[1]
\STATE \textbf{Input:} Relaxed solution $\overline{\bm X}\in\Delta_k^N$\;
\STATE \textbf{Initialize:} $\bm X\leftarrow \bm 0_{N\times k}$\;
\STATE Define the auxiliary random matrix $\bm{Z}$ by~\eqref{eq:aux-random-assignment}\;
\FOR{$i=1$ to $N$}
    \STATE \textbf{Set} $\mathcal C_{i-1}$ by~\eqref{eq:conditioning-event}\;
    \STATE \textbf{Set} $s_i^\star\leftarrow \arg\min_{s\in[k]}\psi_i(s)$, where $\psi_i(s)$ is defined in~\eqref{eq:conditional-cost}\;
    \STATE \textbf{Set} $\bm X_{i\cdot}\leftarrow \bm e_{s_i^\star}$\;
\ENDFOR
\STATE \textbf{Output:} $\bm X\in\mathcal X$\;
\end{algorithmic}
\end{algorithm}

\begin{theorem}
Let $\overline{\bm{X}}$ and $\bm{X}$ be the input and output of Algorithm~\ref{alg:conditional} respectively. Then 
\begin{align}
    f(\bm{X}) \leq f(\overline{\bm{X}}).
\end{align}
\label{thm:conditional}
\end{theorem}
\begin{proof}
See Appendix~\ref{proof:thm:conditional}. 
\end{proof}

This conditional expectation rounding deterministically produces a discrete feasible solution whose objective value is guaranteed not to exceed that of the relaxed solution. Our method extends the conditional expectation method from Max-Cut~\cite{karalias2020erdos, prabhakar1988probabilistic} to weighted Min-$k$-Partition with simplex-relaxed variables. 
While the basic implementation of Algorithm~\ref{alg:conditional} evaluates conditional costs independently for each node, leading to a naive complexity of $\mathcal{O}(k^2N\vert \mathcal{E}\vert)$, we adopt an incremental update scheme that maintains edge-local expectations and updates them only when a node is fixed. 
This reduces the overall rounding complexity to $\mathcal{O}(k|\mathcal E|)$.
The complete incremental procedure and analysis are given in Appendix~\ref{app:conditional}.

\subsection{Edge-Based Refinement}

After conditional expectation rounding, we further apply a monotone local refinement step. The procedure is summarized in Algorithm~\ref{alg:refine}.
For each weighted edge $(i,j)$, the algorithm enumerates all label pairs $(c_1,c_2)\in[k]^2$ and accepts the reassignment 
\[
(\bm X_{i\cdot},\bm X_{j\cdot})
\leftarrow
(\bm e_{c_1},\bm e_{c_2})
\]
only if it decreases the objective. 
The procedure repeats until no improving reassignment is found or a time budget (e.g., $30$ seconds) is reached. 
Since every accepted move decreases the objective, this refinement preserves the guarantee obtained after rounding while further improving the final discrete solution.

\begin{algorithm}[t]
\caption{Edge-based Local Search Refinement}
\label{alg:refine}
\KwData{Initial labels $\bm{X}$ from Algorithm~\ref{alg:conditional} of GPO, edge set $\mathcal{E}$ with weights $\bm{W}$, number of partitions $k$}
\KwResult{Refined labels $\bm{X}$}

\Repeat{no improvement or time budget reached}{
    \For{each $(i,j,w)\in \mathcal{E}$}{
        \For{$c_1=1$ to $k$}{
            \For{$c_2=1$ to $k$}{
                Compute cost difference $\Delta f$ for $(\bm{X}_{i\cdot},\bm{X}_{j\cdot})\leftarrow (\bm{e}_{c_1},\bm{e}_{c_2})$\;
                \If{$\Delta f < 0$}{
                    Update $(\bm{X}_{i\cdot},\bm{X}_{j\cdot})\leftarrow (\bm{e}_{c_1},\bm{e}_{c_2})$\;
                }
            }
        }
    }
}
\KwRet{$\bm{X}$}
\end{algorithm}

\subsection{Summary}

We summarize the proposed GNN-parametrized optimization framework for solving each weighted Min-$k$-Partition block. 
The method consists of three steps:
\begin{enumerate}
    \item \textbf{Instance-wise neural optimization:} optimize the GNN parameters on the current weighted block graph by minimizing the objective~\eqref{eq:gnn-parametrization}.
    \item \textbf{Deterministic rounding:} convert the relaxed solution into a discrete assignment via conditional expectation rounding, which does not increase the objective value.
    \item \textbf{Edge-based refinement:} further improve the rounded assignment by a monotone local search over interference-relevant edges.
\end{enumerate}
The full procedure is outlined in Algorithm~\ref{alg:GNN_pipeline}.

\begin{algorithm}[t]
\caption{GNN-Parametrized Optimization for Min-$k$-Partition}
\label{alg:GNN_pipeline}
\KwData{Weighted graph matrix $\bm{W}$ and number of partitions $k$.}
\KwResult{Discrete assignment $\bm{X}\in\mathcal{X}$.}

\BlankLine
\textbf{Initialize} GNN parameters $\bm{\Phi}$\;
\Repeat{convergence}{
    \textbf{Update} embeddings $\bm{H}^{(L)}_{\bm{\Phi}}$ using GNN layers\;
    \textbf{Compute loss} $\mathcal{L}(\bm{\Phi}) \gets \mathrm{Tr}(\bm{H}^{(L)\top}_{\bm{\Phi}}\bm{W}\bm{H}^{(L)}_{\bm{\Phi}})$\;
    \textbf{Optimize} $\bm{\Phi}$ via gradient-based training\;
}
\textbf{Obtain} the optimized parameter $\bm{\Phi}^\star$\;
\textbf{Compute} relaxed solution $\overline{\bm{X}} \gets \bm{H}^{(L)}_{\bm{\Phi}^\star}$\; 
\textbf{Apply} Algorithm~\ref{alg:conditional} to obtain $\bm{X} \in \mathcal{X}$\;
\textbf{Apply} Algorithm~\ref{alg:refine} to refine solution $\bm{X} \in\mathcal{X}$\;
\BlankLine
\textbf{Output:} $\bm{X}$
\end{algorithm}

Compared with the PMD approach~\cite{qiu2025relaxation}, our framework differs in both relaxation and optimization. 
PMD reformulates Min-$k$-Partition as a quadratic program with norm-equality constraints and enforces feasibility through a penalty-based iterative procedure, which requires an outer loop over the penalty parameter. 
In contrast, we directly optimize the penalty-free simplex relaxation~\eqref{eq:soft-relaxation} through a GNN-based neural parametrization. 
The resulting relaxed solution is then converted into a discrete assignment by conditional expectation rounding, which deterministically guarantees that the rounded objective value is no larger than the relaxed objective value. 

By avoiding penalty tuning and introducing a graph-aware parametrization of the relaxed assignment, the proposed solver provides a simpler alternative for large-scale block optimization. 
Furthermore, unlike the naive pairwise local search used in~\cite{qiu2025relaxation}, our refinement is edge-restricted and only examines interference-relevant pairs. 
Since each accepted move decreases the objective, this refinement preserves monotonic improvement while substantially reducing the local-search overhead.

\section{Numerical Experiments \label{sec:exp}}

\begin{figure*}
    \centering
    \includegraphics[width=\linewidth]{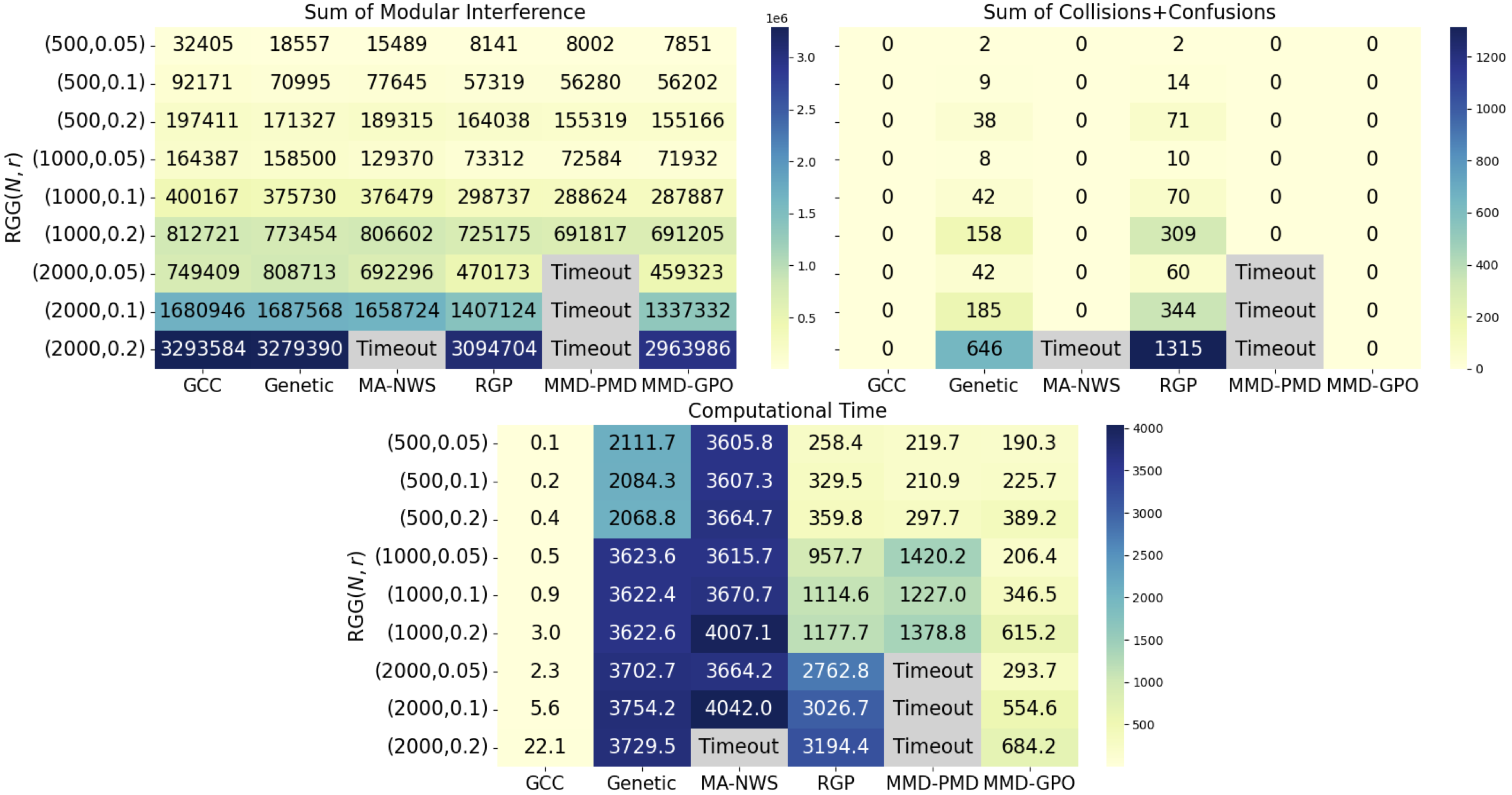}
    \caption{The Evaluation Results on $\mathrm{RGG}(N, r)$ with $N\in\{500, 1000, 2000\}$ and $r\in\{0.05, 0.1, 0.2\}$ of different scales.}
    \label{fig:exp_results}
\end{figure*}

In this section, we evaluate the performance and efficiency of the proposed framework for large-scale PCI assignment. We compare it against state-of-the-art baselines in jointly optimizing mod-$3$, mod-$4$, mod-$6$, and mod-$30$ interference listed in Table~\ref{tab:common-k-extended}, while eliminating collisions and confusions. All four interference types are treated equally, indicating that $\lambda_1=\lambda_2=\lambda_3=\lambda_4=1$. Experiments are implemented in Python on a system equipped with an Intel i9-12900K CPU and an NVIDIA GeForce RTX 3090 GPU.

\subsection{List of Baseline Methods}

We call our proposed method of Multi-Modular Decomposition with GNN Parametrization-based Optimization \texttt{MMD-GPO}. The baseline methods for comparison are described as follows. We select five representative methods for solving the PCI problem.

\textbf{Greedy Graph Coloring Method (\texttt{GGC})} \cite{bandh2009graph}: This method employs a greedy coloring algorithm to address collisions and confusions. In the simulation, we utilize the codes provided from \cite{coloring}.

\textbf{Genetic Method (\texttt{Genetic})}~\cite{panxing2016pci}: We apply a genetic method for comparison. The fitness function is set as the summation of all modular interferences. The hyperparameter settings include a population size of $30$ and $200$ iterations, and the probabilities of crossover, mutation, and selection are set to $0.77$, $0.3$, and $0.8$, respectively. 

\textbf{Memetic Algorithm (\texttt{MA-NWS})}~\cite{andrade2022physical}: We adopt the non-warm-start version as recommended by~\cite{andrade2022physical}, which relies solely on the core evolutionary mechanism of the Memetic Algorithm. The fitness function is set as the summation of all modular interferences. All hyperparameters are set according to the best-performing configuration reported in~\cite{andrade2015biased}.

\textbf{Relaxed Gradient Projection (\texttt{RGP})}~\cite{11149326}:  
We solve the continuous relaxation via gradient projection and apply the rounding rule therein. All parameters, including step sizes and tolerances, are set according to~\cite{11149326}.

\textbf{Multi-Modular Decomposition with Penalized Mirror Descent (\texttt{MMD-PMD})}~\cite{qiu2025relaxation}:  
To isolate the effect of the proposed neural block solver, we combine our congruence decomposition with the PMD block solver of~\cite{qiu2025relaxation}. All hyperparameters, including step sizes, penalty schedule, and iteration limits, are set according to the configurations reported in~\cite{qiu2025relaxation}.

\subsection{Synthetic Data Experiments}

In this section, we test our method on synthetically generated networks.

\subsubsection{Experiment Setup}

\begin{table}[t]
    \centering
    \caption{Structural statistics of the random geometric graphs used in our experiments (averaged over 20 instances per setting).}
    \resizebox{\columnwidth}{!}{%
    \begin{tabular}{ccccc}
        \midrule
        $\mathrm{RGG}(N, r)$ & Density & Max Clique & Avg. Degree & Clustering Coef. \\
        \midrule
        $(500,\ 0.05)$  & $0.01$ & $6.45$  & $3.78$   & $0.53$ \\
        $(500,\ 0.1)$   & $0.03$ & $13.65$ & $14.48$  & $0.62$ \\
        $(500,\ 0.2)$   & $0.11$ & $30.45$ & $52.69$  & $0.65$ \\
        $(1000,\ 0.05)$ & $0.01$ & $9.70$  & $7.54$   & $0.60$ \\
        $(1000,\ 0.1)$  & $0.03$ & $20.45$ & $28.85$  & $0.62$ \\
        $(1000,\ 0.2)$  & $0.11$ & $53.05$ & $105.29$ & $0.65$ \\
        $(2000,\ 0.05)$ & $0.01$ & $14.90$ & $15.09$  & $0.60$ \\
        $(2000,\ 0.1)$  & $0.03$ & $33.75$ & $57.74$  & $0.62$ \\
        $(2000,\ 0.2)$  & $0.11$ & $95.15$ & $210.51$ & $0.65$ \\
        \midrule
    \end{tabular}
    }
    \label{tab:rgg_stats}
\end{table}

Following~\cite{qiu2025relaxation}, we generate synthetic networks using the random geometric graph model $\mathrm{RGG}(N, r)$, where $N$ denotes the number of cells and $r$ is called the connection radius. Cells are uniformly distributed in a 2D unit square, and cell $j$ is treated as a neighbor of node $i$ if their Euclidean distance is at most $r$. The interference matrix $\bm{W}$ is defined by setting the interference inversely proportional to the Euclidean distance, simulating spatial interference in wireless networks. This model is widely used to emulate realistic wireless topologies~\cite{Gowaikar2006, Haenggi2009, qiu2025relaxation}.

To evaluate algorithm performance under diverse topological conditions, we systematically vary both the network size and spatial density. Specifically, we consider network sizes $N \in \{500, 1000, 2000\}$, and connection radii $r \in \{0.05, 0.1, 0.2\}$ which correspond to sparse, medium, and dense graphs, respectively. For each configuration $\mathrm{RGG}(N, r)$, we generate $20$ independent instances, which yields a total of $180$ graph instances. For each type of graph, we compute standard structural metrics, including edge density, maximum clique size, average degree, and clustering coefficient. These statistics are summarized in Table \ref{tab:rgg_stats} and allow us to interpret algorithm performance under varying structural complexities. 
All methods are run with a $3600$s time limit. If an algorithm has not terminated by then, we record its intermediate solution if available; otherwise, it is marked as ``Timeout''.

\texttt{GPO} is designed as a two-layer GNN, with both the input and hidden dimensions set to $5\times k$ for solving Min-$k$-Partition. To address the issue of gradient vanishing, we apply graph normalization as proposed by \cite{cai2021graphnorm}.
The \texttt{GPO} model is optimized using Adam with a learning rate of $10^{-2}$ for a maximum epoch number of $10^{5}$, applying early stopping with a tolerance of $10^{-2}$ and patience of $100$ iterations. Optimization terminates if no improvement is observed.

\subsubsection{Experiment Results} 
We list the performance of five baselines and our proposed method under three key metrics: (i) the \emph{sum of modular interference} (including mod-$3$, mod-$4$, mod-$6$, and mod-$30$ interferences), (ii) the \emph{sum of collisions and confusions}, and (iii) the \emph{computational time}. Note that since GPO optimizes an instance-specific GNN parametrization for each test graph, its per-instance optimization cost is fully included in the reported runtime. The averaged results over $20$ instances are summarized in Fig.~\ref{fig:exp_results}. 

\textbf{Interference Minimization.}  
As shown in the top-left heatmap, both \texttt{MMD-PMD} and \texttt{MMD-GPO} achieve substantially lower modular interference compared with all other baselines, which validates the effectiveness of the proposed MMD in dealing with this multi-objective challenge. Furthermore, \texttt{MMD-GPO} consistently outperforms \texttt{MMD-PMD}, demonstrating that GPO provides an additional gain over the plain PMD strategy by further refining the interference reduction.  

\textbf{Collision and Confusion Reduction.}  
The top-right heatmap shows the total number of collisions and confusions. Both \texttt{MMD-PMD} and \texttt{MMD-GPO} completely eliminate such conflicts in all settings, confirming that the MMD framework guarantees collision- and confusion-free assignments in these cases. Although \texttt{GCC} also returns zero collisions and confusions, its solution ignores modular interference altogether and is therefore not practically applicable. In contrast, the remaining heuristics suffer from persistent residual conflicts, especially as the network size increases.  

\textbf{Computational Efficiency.}  
The bottom heatmap illustrates the average computational time. 
MMD-GPO exhibits a clear scalability advantage as the network size increases. For $N\ge 1000$, it is substantially faster than MMD-PMD across all completed settings, while MMD-PMD times out on the largest instances, highlighting the scalability advantage of GPO. 
Although GGC has the shortest runtime, it only addresses collisions and confusions without optimizing modular interference.
Overall, \texttt{MMD-GPO} provides the best balance between efficiency and performance. 

\textbf{Summary.}  
In summary, the MMD framework is key to eliminating collisions and confusions, while the GPO significantly boosts both performance and efficiency in optimizing multi-modular interference. As a result, \texttt{MMD-GPO} consistently outperforms all modular-interference-aware baselines, demonstrating the robustness and practicality of our approach for large-scale PCI assignment.

\subsection{Real-World Data Experiments}

\subsubsection{Experimental Setup}

We conduct experiments on real measurement-report (MR) data collected from a commercial downlink cellular network in Beijing, China. The dataset contains approximately $2{,}000$ reconfigurable cells deployed in a dense urban area. The MRs provide neighbor relations, frequency-domain co-occurrence statistics, and interaction strengths among cells. Following the preprocessing and graph-construction pipeline established in \cite{qiu2025relaxation}, the MR data is converted into the collision set $\mathcal{E}_1$, the confusion set $\mathcal{E}_2$, and the interference matrix $\bm{W}^{\mathrm{inter}}$. 

\subsubsection{Experimental Results}

Table~\ref{tab:res} reports the performance of all methods on the real-world dataset containing $2{,}000$ reconfigurable cells. We evaluate each method using three metrics:
(i) the total number of collisions and confusions;
(ii) the overall modular interference value; and
(iii) the computational time.

Among the traditional baselines, \texttt{GGC} removes all collisions and confusions but yields the highest interference due to fully discarding modular interference. \texttt{Genetic}, \texttt{MA-NWS}, and \texttt{RGP} reduce interference to some extent, yet none of them attains collision-free and confusion-free assignments, and all incur substantially higher runtime.

Built on the \texttt{MMD} framework, both \texttt{MMD-PMD} and the proposed \texttt{MMD-GPO} generate solutions that strictly satisfy collision-free and confusion-free constraints while achieving considerably lower interference than all classical baselines. Moreover, \texttt{MMD-GPO} reaches the lowest interference overall, improving upon \texttt{MMD-PMD} by $2.8\%$ and outperforming the other modular-interference-aware baselines by $3.4\%\sim 53\%$. It also delivers the best computational efficiency among methods that explicitly optimize modular interference, running about $2\times$ faster than \texttt{MMD-PMD} and significantly faster than other modular-interference-aware baselines. 

In summary, \texttt{MMD-GPO} achieves the most favorable balance between interference minimization and runtime efficiency, while rigorously preserving collision-free and confusion-free feasibility.

\begin{table}[t]
\centering
\resizebox{\linewidth}{!}{
\begin{threeparttable}
\caption{The evaluation results on the real-world data. }
\begin{tabular}{c|ccc}
\hline
Methods                                               & Coll. \& Conf. & Inter. & Time (s)         \\ \hline
\texttt{GGC} \cite{coloring, bandh2009graph} & \textbf{0}                       & 122392610                   & 3.20             \\ \hline
\texttt{Genetic} \cite{panxing2016pci}       & 226                              & 111035257                   & 3649.26          \\
\texttt{MA-NWS} \cite{andrade2022physical}   & 128                              & 62123830                    & 7520.81          \\
\texttt{RGP} \cite{11149326}           & 372                              & 53833127                    & 2423.40          \\ \hline
\texttt{MMD-PMD} \cite{qiu2025relaxation}    & \textbf{0}                       & 53505001                    & 2969.44          \\
\texttt{MMD-GPO}   & \textbf{0}                       & \textbf{51999927}           & \textbf{1452.26} \\ \hline
\end{tabular}
\label{tab:res}
\end{threeparttable}
}
\end{table}

\section{Conclusion \label{sec:con}}

This paper developed a congruence decomposition framework with neural block solvers for large-scale PCI assignment with multiple modular interference objectives. By exploiting the arithmetic structure of modular congruence relations, the proposed decomposition framework transforms the modular assignment stage into a sequence of blockwise Min-\(k\)-Partition subproblems and separates it from collision and confusion handling through graph coloring. The resulting block-coordinate procedure admits finite-step termination under monotone block updates. To improve the scalability of the resulting NP-hard subproblems, we further developed an instance-wise GNN-parametrized block solver with deterministic conditional expectation rounding and edge-based refinement. Experiments on synthetic cellular graphs and real-world 5G network data demonstrate substantial reductions in modular interference and improved computational scalability while maintaining collision- and confusion-free solutions in the evaluated instances. Future work includes developing a deeper theoretical understanding of neural-parametrized block solvers and exploring whether the proposed congruence decomposition can be extended beyond PCI assignment.

\appendices

\section{Proof of Theorem \ref{thm:modular-block-equivalence} \label{proof:lamme:modular-block-equivalence}}

\begin{proof}

    From the CRT, the congruence 
    \begin{align*}
        r_i\equiv r_j \pmod{p_\ell}
    \end{align*}
    holds if and only if, for every prime factor $s_u$, 
    \begin{align*}
        r_i\equiv r_j \pmod{s_u^{t_{u,\ell}}}.
    \end{align*}
    Fix $u\in[n]$ and write the $s_u$-adic expansion with sorted exponents $t_{u,(1)}\le\cdots\le t_{u,(m)}=t_u$: 
    \begin{align*}
        r_i^{(u)} = r_i \bmod s_u^{t_u}=\sum_{h=1}^mr_i^{(u,h)}s_u^{t_{u,(h-1)}}, \\
        r_j^{(u)} = r_j \bmod s_u^{t_u}=\sum_{h=1}^mr_j^{(u,h)}s_u^{t_{u,(h-1)}}, 
    \end{align*}
    where $t_{u,(0)}\coloneq 0$ and each coefficient $r_i^{(u,h)}$ lies in $\mathbb{Z}_{s_u^{\delta_{u,h}}}$ with $\delta_{u,h}=t_{u,(h)}-t_{u,(h-1)}$. 

    Then, the congruence $r_i\equiv r_j \pmod{s_u^{t_{u,\ell}}}$ is equivalent to equality of the two expansions modulo $s_u^{t_{u,\ell}}$, hence to equality of all base-$s_u$ digits whose positional weights are strictly smaller than $s_u^{t_{u,\ell}}$. Concretely,
    \begin{align*}
        &r_i^{(u)}\equiv r_j^{(u)}\pmod{s_u^{t_{u,\ell}}} \\
        \Leftrightarrow&
        \sum_{h:t_{u,(h)}\le t_{u,\ell}} r_i^{(u,h)}s_u^{t_{u,(h-1)}}
        =
        \sum_{h:\;t_{u,(h)}\le t_{u,\ell}} r_j^{(u,h)}\,s_u^{\,t_{u,(h-1)}}.
    \end{align*}

    By the uniqueness of the $s_u$-adic expansion at these positional weights, the last equality holds if and only if
    \begin{align*}
    r_i^{(u,h)}=r_j^{(u,h)}\quad\text{for every }h\text{ with }t_{u,(h)}\le t_{u,\ell}.
    \end{align*}
    Collecting this condition over all prime factors $s_u$ yields exactly the requirement that $r_i^{(u,h)}=r_j^{(u,h)}$ for every pair $(u,h)\in\mathcal{J}_\ell$. This proves the equivalence.

\end{proof}

\section{Proof of Theorem~\ref{thm:convergence_BCD} \label{proof:thm:convergence_BCD}}

\begin{proof}
Let 
\begin{align*}
\Lambda(\{\bm r^{(u,h)}\})=\sum_{\ell=1}^m\lambda_\ell\sum_{i,j}W^{\mathrm{inter}}_{i,j}\prod_{(u,h)\in\mathcal J_\ell}\mathbbm{1}\{r_i^{(u,h)}=r_j^{(u,h)}\}
\end{align*}
be the weighted sum of modular interferences. Since every block update satisfies $\Lambda_{\rm new}<\Lambda_{\rm old}$, the sequence $\{\Lambda^{(t)}\}$ of objective values produced by the algorithm is monotonically decreasing. Since the overall feasible set is finite, $\Lambda$ can assume only finitely many values. Therefore, the decreasing sequence $\{\Lambda^{(t)}\}$ must stabilize after a finite number of updates, i.e., the algorithm reaches a configuration at which no further accepted block update is possible. This implies termination in a finite number of block updates.
\end{proof}

\section{Proof of Theorem \ref{thm:conditional} \label{proof:thm:conditional}}

\begin{proof}
    
Let $\overline{\bm{X}} \in \Delta_k^N$ be a feasible solution to the relaxed problem (\ref{eq:soft-relaxation}), and let $\bm{X}$ be the output of Algorithm~\ref{alg:conditional}.

Define an auxiliary random matrix $\bm Z\in\mathcal X$ by independently drawing each row according to
\[
\mathbb P(\bm Z_{i\cdot}=\bm e_s)=\overline X_{i,s},
\qquad i\in[N],\ s\in[k].
\]
Since $\bm W$ has zero diagonal, the randomized rounding is unbiased~\cite{qiu2024ros}:
\begin{align}
\mathbb E[f(\bm Z)] = f(\overline{\bm X}).
\label{eq:proof-unbiased-rounding}
\end{align}

Consider the sequential conditional expectation procedure. 
Suppose that before fixing node $i$, the first $i-1$ rows have been fixed as
$\bm X_{1\cdot},\ldots,\bm X_{(i-1)\cdot}$. 
Let
\[
\mathcal C_{i-1}
=
\{\bm Z_{\ell\cdot}=\bm X_{\ell\cdot},\ \ell=1,\ldots,i-1\}.
\]
For each candidate label $s\in[k]$, define
\begin{align}
\psi_i(s)
:=
\mathbb E\!\left[
f(\bm Z)
\,\middle|\,
\mathcal C_{i-1},\ \bm Z_{i\cdot}=\bm e_s
\right].
\end{align}
Algorithm~\ref{alg:conditional} selects
\[
s_i^\star\in\arg\min_{s\in[k]}\psi_i(s)
\]
and sets $\bm X_{i\cdot}=\bm e_{s_i^\star}$.

By the law of total expectation,
\begin{align}
\mathbb E\!\left[f(\bm Z)\mid \mathcal C_{i-1}\right]
&=
\sum_{s=1}^k \overline X_{i,s}\psi_i(s) \notag\\
&\ge
\min_{s\in[k]}\psi_i(s)
=
\psi_i(s_i^\star).
\end{align}
Thus, fixing the $i$-th row according to Algorithm~\ref{alg:conditional} does not increase the conditional expected objective.

Applying this argument sequentially for $i=1,\ldots,N$, we obtain
\begin{align}
f(\bm X)
\le
\mathbb E[f(\bm Z)]
=
f(\overline{\bm X}),
\end{align}
where the equality follows from~\eqref{eq:proof-unbiased-rounding}. 
Therefore, Algorithm~\ref{alg:conditional} returns a discrete assignment whose objective value is no larger than that of the relaxed solution.

\end{proof}

\section{Efficient Implementation of Algorithm~\ref{alg:conditional} \label{app:conditional}}

Here we describe an efficient, incremental implementation of Algorithm~\ref{alg:conditional} that maps the relaxed solution $\overline{\bm{X}}\in\Delta_k^N$ to a discrete one-hot assignment $\bm{X}\in\mathcal{X}$. Denote the objective evaluated at $\overline{\bm{X}}$ by
\begin{align*}
f(\overline{\bm{X}})
    =\sum_{(i,j)\in\mathcal{E}} W_{i,j}\,\overline{\bm{X}}_{i\cdot}^{\top}\overline{\bm{X}}_{j\cdot}
    \coloneq \sum_{(i,j)\in\mathcal{E}} W_{i,j}T_{i,j},
\end{align*}
where $T_{i,j}=\overline{\bm{X}}_{i\cdot}^{\top}\overline{\bm{X}}_{j\cdot}$ is the expected contribution of edge $(i,j)$ under the product distribution induced by $\overline{\bm{X}}$. The rounding proceeds by fixing nodes sequentially: when considering node $i$, we evaluate, for each candidate label $s\in\{1,\dots,k\}$, the change in the conditional expected objective obtained by assigning $s$ to node $i$. Since only edges incident to $i$ are affected, it suffices to inspect $\mathcal{N}(i)=\{j\mid (i,j)\in\mathcal{E}\}$. For an incident edge $(i,j)$, the post-conditioning value $T_{i,j}^{(s)}$ is defined as
\begin{align}
T_{i,j}^{(s)}=
\begin{cases}
1, & \text{if node } j \text{ is already fixed to label } s,\\[2pt]
0, & \text{if node } j \text{ is already fixed to a label }\neq s,\\[2pt]
\overline{X}_{j,s}, & \text{if node } j \text{ is not yet fixed}.
\end{cases} \label{eq:incre}
\end{align}
Thus, the incremental change for candidate $s$ is
\begin{align*}
\Delta f_i(s)=\sum_{j\in\mathcal{N}(i)} W_{i,j}\bigl(T_{i,j}^{(s)}-T_{i,j}\bigr).
\end{align*}
We choose the label $s^\star=\arg\min_s \Delta f_i(s)$, set $\bm{X}_{i\cdot}$ to the corresponding one-hot vector, and update $T_{i,j}\leftarrow T_{i,j}^{(s^\star)}$ for all $j\in\mathcal{N}(i)$. This edge-local maintenance avoids recomputing the quadratic form at each step and yields the practical procedure summarized in Algorithm~\ref{appendix:incremental_rounding}.

\begin{algorithm}[t]
\caption{Incremental Conditional Rounding}
\label{appendix:incremental_rounding}
\KwIn{Relaxed solution $\overline{\bm{X}}\in\Delta_k^N$, edge set $\mathcal{E}$, weights $W_{i,j}$.}
\KwOut{Discrete assignment $\bm{X}\in\mathcal{X}$.}

\textbf{Initialize} $T_{i,j}\leftarrow \overline{\bm{X}}_{i\cdot}^{\top}\overline{\bm{X}}_{j\cdot}$ for all $(i,j)\in\mathcal{E}$\;

\For{$i=1$ \KwTo $N$}{
    \For{$s=1$ \KwTo $k$}{
        \textbf{Compute} $T_{i,j}^{(s)}$ for all $j\in\mathcal{N}(i)$ by \eqref{eq:incre}\;
        \textbf{Compute} $\Delta f_i(s)\leftarrow \sum_{j\in\mathcal{N}(i)} W_{i,j}\bigl(T_{i,j}^{(s)}-T_{i,j}\bigr)$\;
    }
    \textbf{Set} $s^\star\leftarrow\arg\min_s \Delta f_i(s)$\;
    \textbf{Set} $\bm{X}_{i\cdot}\leftarrow\bm{e}_{s^\star}$\;
    \textbf{Update} $T_{i,j}\leftarrow T_{i,j}^{(s^\star)}$ for all $j\in\mathcal{N}(i)$\;
}
\Return $\bm{X}$\;
\end{algorithm}

The initial computation of $f(\overline{\bm{X}})$ requires $O(k\vert\mathcal{E}\vert)$ time. During rounding, for each node $i$ and each label $s$, we inspect all incident edges, and the total number of such edge visits over the entire process is exactly $\vert\mathcal{E}\vert$. Hence, the overall time complexity is $O(k\vert\mathcal{E}\vert)$.

\bibliographystyle{IEEEtran}
\bibliography{Reference}

@InProceedings{xavier2010understanding,
  title = 	 {Understanding the difficulty of training deep feedforward neural networks},
  author = 	 {Glorot, Xavier and Bengio, Yoshua},
  booktitle = 	 {Proceedings of the Thirteenth International Conference on Artificial Intelligence and Statistics},
  pages = 	 {249--256},
  year = 	 {2010},
  editor = 	 {Teh, Yee Whye and Titterington, Mike},
  volume = 	 {9},
  series = 	 {Proceedings of Machine Learning Research},
  address = 	 {Chia Laguna Resort, Sardinia, Italy},
  month = 	 {13--15 May},
  publisher =    {PMLR},
}

@INPROCEEDINGS{11149326,
  author={Qiu, Yeqing and Xue, Ye and Jiang, Zhipeng and Shi, Qingjiang},
  booktitle={2025 IEEE/CIC International Conference on Communications in China (ICCC)}, 
  title={Relaxed Gradient Projection for {PCI} Assignment in {5G} Network}, 
  year={2025},
  volume={},
  number={},
  pages={1-6},
  doi={10.1109/ICCC65529.2025.11149326}}

@ARTICLE{9962800,
  author={Zhao, Zhongyuan and Verma, Gunjan and Rao, Chirag and Swami, Ananthram and Segarra, Santiago},
  journal={IEEE Transactions on Wireless Communications}, 
  title={Link Scheduling Using Graph Neural Networks}, 
  year={2023},
  volume={22},
  number={6},
  pages={3997-4012},
  doi={10.1109/TWC.2022.3222781}}

@article{schuetz2022combinatorial,
  title={Combinatorial optimization with physics-inspired graph neural networks},
  author={Schuetz, Martin JA and Brubaker, J Kyle and Katzgraber, Helmut G},
  journal={Nature Machine Intelligence},
  volume={4},
  number={4},
  pages={367--377},
  year={2022},
  publisher={Nature Publishing Group UK London}
}

@ARTICLE{Haenggi2009,
  author={Haenggi, Martin and Andrews, Jeffrey G. and Baccelli, Francois and Dousse, Olivier and Franceschetti, Massimo},
  journal={IEEE J. Sel. Areas Commun.},
  title={Stochastic geometry and random graphs for the analysis and design of wireless networks}, 
  year={2009},
  volume={27},
  number={7},
  pages={1029-1046},
  doi={10.1109/JSAC.2009.090902}}

@ARTICLE{Gowaikar2006,
  author={Gowaikar, R. and Hochwald, B. and Hassibi, B.},
  journal={IEEE Trans. Inf. Theory},
  title={Communication over a wireless network with random connections}, 
  year={2006},
  volume={52},
  number={7},
  pages={2857-2871},
  doi={10.1109/TIT.2006.876254}}

@article{andrade2015biased,
title = {A biased random-key genetic algorithm for wireless backhaul network design},
  journal={Appl. Soft Comput.},
volume = {33},
pages = {150-169},
year = {2015},
issn = {1568-4946},
doi = {https://doi.org/10.1016/j.asoc.2015.04.016},
author = {Carlos E. Andrade and Mauricio G.C. Resende and Weiyi Zhang and Rakesh K. Sinha and Kenneth C. Reichmann and Robert D. Doverspike and Flávio K. Miyazawa},
}

@article{acedo2015analysis,
  title={Analysis of the impact of {PCI} planning on downlink throughput performance in {LTE}},
  author={Acedo-Hern{\'a}ndez, Roc{\'\i}o and Toril, Matias and Luna-Ram{\'\i}rez, Salvador and de la Bandera, Isabel and Faour, N},
  journal={Computer Networks},
  volume={76},
  pages={42--54},
  year={2015},
  publisher={Elsevier}
}

@article{zeljkovic2022alpaca,
  title={{ALPACA}: A {PCI} Assignment Algorithm Taking Advantage of Weighted {ANR}},
  author={Zeljkovi{\'c}, Ensar and Gogos, Dimitris and Dox, Gerwin and Latr{\'e}, Steven and Marquez-Barja, Johann M},
  journal={Journal of Network and Systems Management},
  volume={30},
  number={2},
  pages={33},
  year={2022},
  publisher={Springer}
}

@misc{coloring,
  author       = {Guillem G. Subies},
  title        = {Welsh Powell Algorithm to color graphs },
  howpublished = {\url{https://gist.github.com/GuillemGSubies/59b968ec0a68c17384c9f15c15cb38e6}},
}

@article{prabhakar1988probabilistic,
title = {Probabilistic construction of deterministic algorithms: Approximating packing integer programs},
journal = {Journal of Computer and System Sciences},
volume = {37},
number = {2},
pages = {130-143},
year = {1988},
issn = {0022-0000},
doi = {https://doi.org/10.1016/0022-0000(88)90003-7},
url = {https://www.sciencedirect.com/science/article/pii/0022000088900037},
author = {Prabhakar Raghavan}
}

@article{fairbrother2018two,
  title={A two-level graph partitioning problem arising in mobile wireless communications},
  author={Fairbrother, Jamie and Letchford, Adam N and Briggs, Keith},
  journal={Comput. Optim. and Appl.},
  volume={69},
  pages={653--676},
  year={2018},
  publisher={Springer}
}

@inproceedings{shen2017novel,
  title={A novel {PCI} optimization method in {LTE} system based on intelligent genetic algorithm},
  author={Shen, Ao and Guo, Bao and Gao, Yan and Xie, Tao and Hu, Xiaochun and Zhang, Yang and Shen, Jinhu and Fang, Yuan and Wang, Guozhi and Liu, Yi},
  booktitle={Int. Conf. Signal and Inf. Process., Netw. And Comput.},
  pages={350--355},
  year={2017},
  organization={Springer}
}

@article{pratap2016randomized,
  title={Randomized graph coloring algorithm for physical cell {ID} assignment in {LTE}-a femtocellular networks},
  author={Pratap, Ajay and Misra, Rajiv and Gupta, Utkarsh},
  journal={Wireless Pers. Commun.},
  volume={91},
  pages={1213--1235},
  year={2016},
  publisher={Springer}
}

@inproceedings{bandh2009graph,
author = {Bandh, Tobias and Carle, Georg and Sanneck, Henning},
title = {Graph coloring based physical-cell-{ID} assignment for {LTE} networks},
year = {2009},
isbn = {9781605585697},
publisher = {Association for Computing Machinery},
doi = {10.1145/1582379.1582406},
booktitle={Proc. of the 2009 Int. Conf. Wireless Commun. and Mobile Comput.: Connecting the world wirelessly},
pages = {116–120},
numpages = {5},
location = {Leipzig, Germany},
series = {IWCMC '09}
}

@misc{3gpp38211,
  organization = "{3rd Generation Partnership Project (3GPP)}",
  title        = "{3GPP TS 38.211 V17.4.0: NR; Physical channels and modulation}",
  year         = "2024"
}

@Article{lodhi2023design,
AUTHOR = {Lodhi, Khalid and Chhillar, Jayant and Darak, Sumit J. and Sharma, Divisha},
TITLE = {Design and Performance Analysis of Hardware Realization of {3GPP} Physical Layer for {5G} Cell Search},
JOURNAL = {Chips},
VOLUME = {2},
YEAR = {2023},
NUMBER = {4},
PAGES = {223--242},
ISSN = {2674-0729},
DOI = {10.3390/chips2040014}
}

@ARTICLE{pierucci2015quality,
  author={Pierucci, Laura},
  journal={IEEE Wireless Commun.},
  title={The quality of experience perspective toward {5G} technology}, 
  year={2015},
  volume={22},
  number={4},
  pages={10-16},
  doi={10.1109/MWC.2015.7224722}}

@ARTICLE{liu2024,
  author={Liu, Ya-Feng and Chang, Tsung-Hui and Hong, Mingyi and Wu, Zheyu and Man-Cho So, Anthony and Jorswieck, Eduard A. and Yu, Wei},
  journal={IEEE J. Sel. Areas Commun.},
  title={A Survey of Recent Advances in Optimization Methods for Wireless Communications}, 
  year={2024},
  volume={42},
  number={11},
  pages={2992-3031},
  doi={10.1109/JSAC.2024.3443759}
}

@article{luo2023srcon,
  author={Luo, Zhi-Quan and Zheng, Xi and López-Pérez, David and Yan, Qi and Chen, Xin and Wang, Nanbin and Shi, Qingjiang and Chang, Tsung-Hui and Garcia-Rodriguez, Adrian},
  journal={IEEE Commun. Mag.},
  title={{SRCON}: A Data-Driven Network Performance Simulator for Real-World Wireless Networks}, 
  year={2023},
  volume={61},
  number={6},
  pages={96-102},
  doi={10.1109/MCOM.001.2200179}
}

@ARTICLE{xu2021survey,
  author={Xu, Yongjun and Gui, Guan and Gacanin, Haris and Adachi, Fumiyuki},
  journal={IEEE Communications Surveys \& Tutorials}, 
  title={A Survey on Resource Allocation for {5G} Heterogeneous Networks: Current Research, Future Trends, and Challenges}, 
  year={2021},
  volume={23},
  number={2},
  pages={668-695},
  doi={10.1109/COMST.2021.3059896}}

@article{luolou2015enhanced,
author = {Loulou, AlaaEddin and Renfors, Markku},
title = {Enhanced {OFDM} for fragmented spectrum use in {5G} systems},
journal = {Transactions on Emerging Telecommunications Technologies},
volume = {26},
number = {1},
pages = {31-45},
doi = {https://doi.org/10.1002/ett.2898},
year = {2015}
}

@ARTICLE{mansoor2017tutorial,
  author={Shafi, Mansoor and Molisch, Andreas F. and Smith, Peter J. and Haustein, Thomas and Zhu, Peiying and De Silva, Prasan and Tufvesson, Fredrik and Benjebbour, Anass and Wunder, Gerhard},
  journal={IEEE Journal on Selected Areas in Communications}, 
  title={{5G}: A Tutorial Overview of Standards, Trials, Challenges, Deployment, and Practice}, 
  year={2017},
  volume={35},
  number={6},
  pages={1201-1221},
  doi={10.1109/JSAC.2017.2692307}}

@ARTICLE{ge2016ultra,
  author={Ge, Xiaohu and Tu, Song and Mao, Guoqiang and Wang, Cheng-Xiang and Han, Tao},
  journal={IEEE Wireless Communications}, 
  title={{5G} Ultra-Dense Cellular Networks}, 
  year={2016},
  volume={23},
  number={1},
  pages={72-79},
  doi={10.1109/MWC.2016.7422408}
}

@InProceedings{qiu2024ros,
  title={{ROS}: A {GNN}-based Relax-Optimize-and-Sample Framework for Max-$k$-Cut Problems},
  author={Qiu, Yeqing and Xue, Ye and Wang, Akang and Wang, Yiheng and Shi, Qingjiang and Luo, Zhi-Quan},
  booktitle={Proceedings of the 42nd International Conference on Machine Learning},
  year={2025}
}

@ARTICLE{qiu2025relaxation,
  author={Qiu, Yeqing and Huang, Chengpiao and Xue, Ye and Jiang, Zhipeng and Shi, Qingjiang and Zhang, Dong and Luo, Zhi-Quan},
  journal={IEEE Trans. Signal Process.}, 
  title={Relaxation-Free Min-$k$-Partition for {PCI} Assignment in {5G} Networks}, 
  year={2025},
  volume={73},
  number={},
  pages={3931-3946},
  doi={10.1109/TSP.2025.3604409}}

@article{karalias2020erdos,
  title={Erdos goes neural: an unsupervised learning framework for combinatorial optimization on graphs},
  author={Karalias, Nikolaos and Loukas, Andreas},
  journal={Advances in Neural Information Processing Systems},
  volume={33},
  pages={6659--6672},
  year={2020}
}

@inproceedings{cai2021graphnorm,
  title={Graphnorm: A principled approach to accelerating graph neural network training},
  author={Cai, Tianle and Luo, Shengjie and Xu, Keyulu and He, Di and Liu, Tie-yan and Wang, Liwei},
  booktitle={International Conference on Machine Learning},
  pages={1204--1215},
  year={2021},
  organization={PMLR}
}

@inproceedings{morris2019weisfeiler,
  title={Weisfeiler and leman go neural: Higher-order graph neural networks},
  author={Morris, Christopher and Ritzert, Martin and Fey, Matthias and Hamilton, William L and Lenssen, Jan Eric and Rattan, Gaurav and Grohe, Martin},
  booktitle={Proceedings of the AAAI conference on artificial intelligence},
  volume={33},
  pages={4602--4609},
  year={2019}
}

@article{andrade2022physical,
  title={The physical cell identity assignment problem: a practical optimization approach},
  author={Andrade, Carlos E and Pessoa, Luciana S and Stawiarski, Slawomir},
  journal={IEEE Transactions on Evolutionary Computation},
  volume={28},
  number={2},
  pages={282--292},
  year={2022},
  publisher={IEEE}
}

@article{panxing2016pci,
  title={{PCI} planning method based on genetic algorithm in {LTE} network},
  author={Li, Panxing and Wang, Jing},
  journal={Telecommunications Science},
  volume={32},
  number={3},
  pages={2016082},
  year={2016}
}

@article{gui2018pci2,
  title={{PCI} planning based on binary quadratic programming in {LTE}/{LTE}-a networks},
  author={Gui, Jihong and Jiang, Zhipeng and Gao, Suixiang},
  journal={IEEE Access},
  volume={7},
  pages={203--214},
  year={2018},
  publisher={IEEE}
}

\vfill

\end{document}